%% file: main.tex
\documentclass{article}
\usepackage[T1]{fontenc}
\usepackage[margin=2.8cm]{geometry}
\usepackage{graphicx}
\graphicspath{{figures/}}
\usepackage{picinpar}
\usepackage{subcaption}
\usepackage{xcolor}
\usepackage{colortbl}
\usepackage{booktabs}
\usepackage{multirow}
\usepackage{enumerate}
\usepackage{enumitem}
\usepackage{pifont}   
\usepackage{array}    

\usepackage{tikz}
\usetikzlibrary{arrows.meta, backgrounds, positioning}
\usepackage{amsmath}
\usepackage{amssymb}
\usepackage{amsthm}
\usepackage{algorithm}
\usepackage{algpseudocode}
\usepackage{wrapfig}      
\usepackage{placeins}     
\usepackage{pgfplots}     
\pgfplotsset{compat=1.18}
\usepgfplotslibrary{groupplots}

\newtheorem{proposition}{Proposition}
\newtheorem{theorem}{Theorem}

\usepackage{hyperref}            
\usepackage{cleveref}            
\usepackage{kpfonts}             
\usepackage{natbib}              
\usepackage{tcolorbox}           
\tcbuselibrary{skins}
\usepackage{fontawesome5}        
\usepackage{xspace}              
\usepackage{microtype}           

\usepackage[]{authblk}
\usepackage{fancyhdr}
\fancypagestyle{firstpage}{%
  \fancyhead[L]{\raisebox{0.05\height}{\hypersetup{hidelinks}\href{https://www.olaresearch.org/}{\includegraphics[height=1.2em]{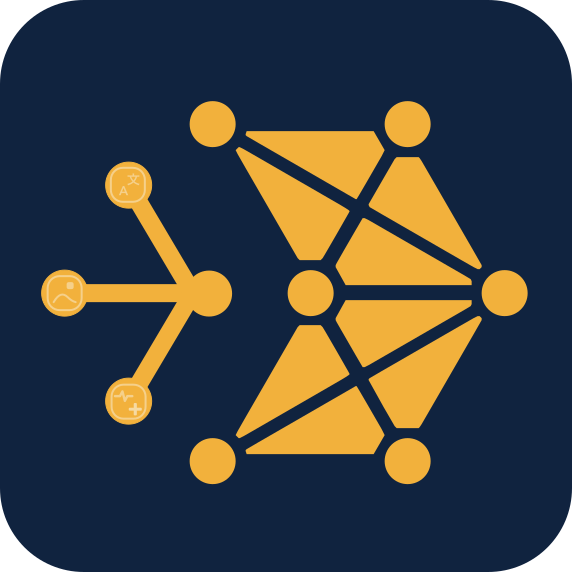}}}~Preprint}%
  \fancyhead[R]{\raisebox{0.05\height}{\hypersetup{hidelinks}\href{https://www.ellisinstitute.fi}{\usebox{\ellislogo}}}}%
}

\newsavebox{\ellislogo}
\sbox{\ellislogo}{\includegraphics[height=1em]{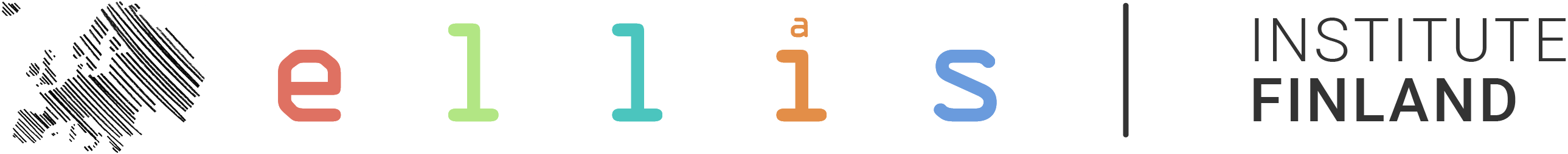}}

\setcitestyle{authoryear,round,citesep={;},aysep={,},yysep={;}}
\definecolor{colornavy}{HTML}{1B365D}      
\definecolor{coloramber}{HTML}{E0913A}     
\definecolor{colorviolet}{HTML}{8B5DCB}    
\definecolor{colormint}{HTML}{004A43}      
\definecolor{colorblue}{HTML}{1D70B8}      
\definecolor{colorforest}{HTML}{194D33}    
\definecolor{colorsage}{HTML}{5F8575}      
\definecolor{colorburgundy}{HTML}{800020}   
\definecolor{colorcrimson}{HTML}{B30D26}   
\definecolor{colorcoral}{HTML}{E07A5F}     
\definecolor{colorindigo}{HTML}{3F51B5}    
\definecolor{colorplum}{HTML}{5C2D50}      
\definecolor{colorslate}{HTML}{475569}     

\definecolor{linkcolor}{rgb}{0.21,0.49,0.74}

\hypersetup{
  colorlinks=false,
  pdfborder={0 0 1},
  linkbordercolor=colornavy,  
  citebordercolor=colorforest, 
  urlbordercolor=colorblue     
}

\newcommand{\posd}[1]{\textcolor{green!40!black}{#1}}

\renewenvironment{abstract}{%
  \begin{tcolorbox}[
    colback=colornavy!4,
    colframe=colornavy,
    leftrule=4pt,
    rightrule=0.5pt,
    toprule=0.5pt,
    bottomrule=0.5pt,
    arc=3pt,
    boxsep=5pt,
    left=12pt,
    right=12pt,
    top=10pt,
    bottom=10pt,
    title=\textbf{Abstract},
    coltitle=colormint,
    attach title to upper,
    after title={\par\medskip}
  ]%
}{%
  \medskip{\color{colornavy}\hrule height 0.5pt}\medskip
  \small\noindent
  {\hypersetup{hidelinks}%
  \begin{tabular}{@{}l@{\hspace{0.5em}}l@{}}
    \makebox[1.15em][c]{\textcolor{colormint}{\faGlobe}} & \textbf{Project Website:} \href{https://www.olaresearch.org/LabelFreeSteering/}{\textit{olaresearch.org/LabelFreeSteering}} \\
    \makebox[1.15em][c]{\textcolor{colormint}{\faGithub}} & \textbf{GitHub Code:} \href{https://github.com/OLAResearch/label-free-bias-steering}{\textit{github.com/OLAResearch/LabelFreeSteering}} \\
    \makebox[1.15em][c]{\raisebox{-0.12\height}{\includegraphics[height=1.2em]{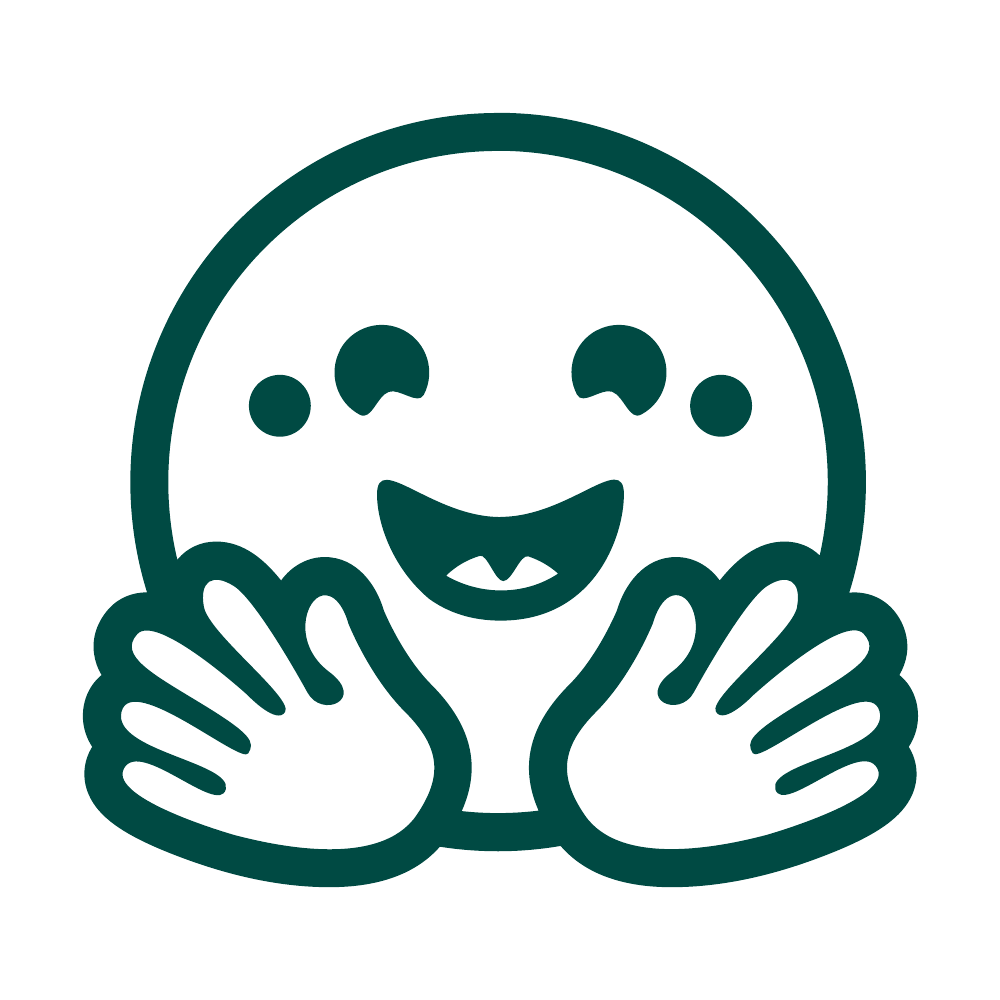}}} & \textbf{HuggingFace Checkpoints:} \href{https://huggingface.co/collections/OLAResearchX/label-free-bias-only-steering}{\textit{hf.co/collections/OLAResearchX/label-free-bias-only-steering}}
  \end{tabular}%
  }
  \end{tcolorbox}%
}

\usepackage[]{lineno}

\title{Label-Free Steering: Compressing Test-Time Reinforcement Learning into Bias-Only Subspaces}

\author[$\dag$]{Naveen Vakada}
\author[$\S$,$\dag$]{Mingyuan Li}
\author[$\S$,$\dag$]{Shaoxiong Ji}
\affil[$\dag$]{University of Turku, Finland}
\affil[$\S$]{ELLIS Institute Finland}
\affil[ ]{\textit{Emails: \texttt{\{navaka,mingyl,shaoxiong.ji@utu.fi\}}}}

\date{}

\begin{document}

\maketitle
\thispagestyle{firstpage}

\begin{abstract}
Test-time reinforcement learning (TTRL) enables models to improve their reasoning without relying on labeled training data, but existing approaches typically optimize a large fraction of the model parameters.
This raises a natural question: can effective test-time adaptation emerge when both the reward signal and the optimization space are severely restricted?
We answer this question with \textit{label-free bias-only TTRL}, which uses majority-vote pseudo-labels as rewards and optimizes only $\sim$100K bias parameters while keeping the pretrained backbone frozen.
On MATH-500, our approach reaches 76.67\% accuracy with Qwen2.5-7B, slightly exceeding our own labeled bias-steering reproduction while optimizing 76,000$\times$ fewer parameters than full-parameter TTRL.
The same training procedure improves performance across vision-language and audio reasoning tasks, including MathVista, AI2D, LogicVista, and MMAU.
We further show that the learned steering vectors transfer to 4,500 held-out MATH problems, indicating that the adaptation is not limited to the problems used during test-time optimization.
Finally, we analyze why this highly restricted adaptation can work, showing that majority-vote reliability improves with rollout consensus and that bias subspaces with greater accessible gradient energy exhibit stronger downstream trainability.
These results demonstrate that substantial test-time adaptation can emerge from optimizing a tiny bias-only subspace using entirely label-free rewards.
\end{abstract}

\section{Introduction}
\label{sec:introduction}

Test-time learning aims to improve a model using only information available at
deployment, without requiring additional labeled training data.
Recent test-time reinforcement learning (TTRL) methods show that useful
optimization signals can be constructed from a model's own generations, for
example through agreement among multiple reasoning rollouts
\citep{zuo2025ttrl,wang2022selfconsistency}.
However, existing approaches typically perform this adaptation in a large
parameter space, often updating most or all of the model.
This leaves a basic question unresolved:
\emph{how much optimization capacity is actually necessary for self-generated
test-time learning?}

We study this question under an extreme optimization bottleneck.
The pretrained backbone is frozen, while adaptation is restricted to only
$\sim$100K trainable parameters in a 7B-scale model.
At the same time, no ground-truth reward is available: the learning signal must
be inferred entirely from the model's own test-time rollouts.
In this regime, a noisy self-generated signal must induce useful behavioral
change through only a minute fraction of the model's parameter space.
We therefore treat parameter efficiency not merely as a computational objective,
but as a way to study what makes restricted test-time adaptation possible.


Our experiments reveal that substantial adaptation survives this bottleneck,
but only for suitable trainable subspaces.
As shown in Figure~\ref{fig:subspace-comparison}, restricting all methods to
approximately the same 100K-parameter budget does not produce comparable
behavior: bias-only adaptation consistently improves performance across all six
evaluated model--benchmark settings, whereas parameter-matched alternatives can
be substantially less effective.
The key variable is therefore not simply the number of trainable parameters,
but the optimization directions exposed by those parameters.

\begin{wrapfigure}{r}{0.49\textwidth}
\vspace{-10pt}
\centering

\resizebox{\linewidth}{!}{%
\begin{tikzpicture}[
    title/.style={font=\small\bfseries},
    legend/.style={font=\scriptsize},
    rowlab/.style={font=\small, anchor=east},
    ticklab/.style={font=\scriptsize},
]

\node[title] at (0.25,1.35)
{Same 100{,}352-parameter budget, different subspaces};

\fill[linkcolor!85]
    (-4.10,0.72) rectangle (-3.86,0.88);
\node[legend, anchor=west] at (-3.78,0.80)
    {\textbf{Bias-only (ours)}};

\fill[black!45]
    (-1.72,0.72) rectangle (-1.48,0.88);
\node[legend, anchor=west] at (-1.40,0.80)
    {LoRA (exact-match)};

\fill[orange!70!black]
    (0.92,0.72) rectangle (1.16,0.88);
\node[legend, anchor=west] at (1.24,0.80)
    {ReFT};

\draw[gray!55, dashed, line width=0.45pt]
    (0,0.35) -- (0,-4.05);

\draw[gray!65, line width=0.45pt]
    (-1.80,-4.05) -- (3.55,-4.05);

\foreach \x/\lab in {
    -1.1/$-10$,
     0/$0$,
     1.1/$+10$,
     2.2/$+20$,
     3.3/$+30$
}{
    \draw[gray!65] (\x,-4.05) -- (\x,-4.16);
    \node[ticklab] at (\x,-4.35) {\lab};
}

\node[ticklab] at (0.85,-4.80)
{$\Delta$ accuracy vs.\ unadapted baseline (points)};

\node[rowlab] at (-1.90, 0.00) {MATH-500 (7B)};
\node[rowlab] at (-1.90,-0.72) {MATH-500 (Math-7B)};
\node[rowlab] at (-1.90,-1.44) {AI2D};
\node[rowlab] at (-1.90,-2.16) {LogicVista};
\node[rowlab] at (-1.90,-2.88) {MathVista};
\node[rowlab] at (-1.90,-3.60) {MMAU};


\fill[linkcolor!85]
    (0, 0.115) rectangle (2.585, 0.285);

\fill[black!45]
    (0,-0.085) rectangle (2.508, 0.085);

\fill[orange!70!black]
    (0,-0.285) rectangle (0.616,-0.115);

\fill[linkcolor!85]
    (0,-0.605) rectangle (3.385,-0.435);

\fill[black!45]
    (0,-0.805) rectangle (2.571,-0.635);

\fill[orange!70!black]
    (0,-1.005) rectangle (2.321,-0.835);

\fill[linkcolor!85]
    (0,-1.325) rectangle (0.518,-1.155);

\fill[black!45]
    (0,-1.525) rectangle (0.185,-1.355);

\fill[orange!70!black]
    (0,-1.725) rectangle (-1.460,-1.555);

\fill[linkcolor!85]
    (0,-2.045) rectangle (0.647,-1.875);

\fill[black!45]
    (0,-2.245) rectangle (-0.180,-2.075);

\fill[orange!70!black]
    (0,-2.445) rectangle (-1.563,-2.275);

\fill[linkcolor!85]
    (0,-2.765) rectangle (0.418,-2.595);

\fill[black!45]
    (0,-2.965) rectangle (0.239,-2.795);

\fill[orange!70!black]
    (0,-3.165) rectangle (-0.550,-2.995);

\fill[linkcolor!85]
    (0,-3.485) rectangle (0.382,-3.315);

\fill[black!45]
    (0,-3.685) rectangle (0.074,-3.515);

\fill[orange!70!black]
    (0,-3.885) rectangle (-0.426,-3.715);

\end{tikzpicture}%
}

\caption{
\textbf{Same parameter budget, different trainability.}
Under the same 100{,}352-parameter budget and label-free TTRL recipe,
different trainable subspaces exhibit markedly different behavior.
Bias-only adaptation improves all six model--benchmark settings,
showing that parameter count alone does not determine adaptation quality.
Baselines: LoRA (exact-match), \S\ref{sec:rq1}; bias-free ReFT \citep{wu2024reft}, Appendix~\ref{app:rq2-extra}.
}
\label{fig:subspace-comparison}

\vspace{-10pt}
\end{wrapfigure}
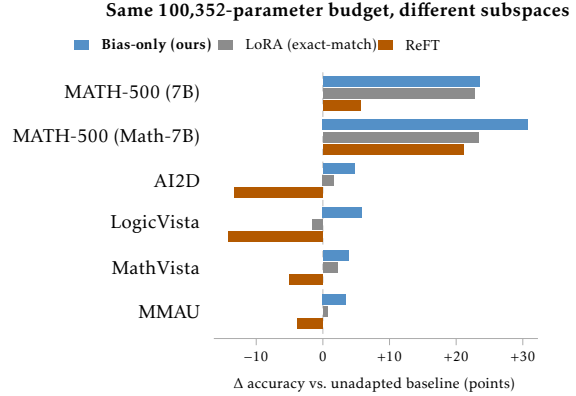
Using additive bias vectors as this restricted subspace, label-free TTRL reaches 76.67\% with Qwen2.5-7B and 79.50\% on MATH-500 with Qwen2.5-Math-7B, while
updating approximately $76{,}000\times$ fewer parameters than full-model
fine-tuning.
The learned intervention also extends beyond the problems used for test-time
optimization: freezing the learned bias vectors and applying them to 4,500
verified-disjoint MATH problems improves Qwen2.5-7B from 46.3\% to 70.9\% and
Qwen2.5-Math-7B from 52.5\% to 75.4\%, without further training.
The same procedure further improves reasoning across vision-language and audio
tasks, including AI2D, LogicVista, MathVista, and MMAU.

These results raise a deeper question:
\emph{why can such a small optimization space support substantial test-time
learning?}
We identify two complementary requirements.
First, the self-generated reward must provide a sufficiently reliable learning
signal.
For majority-vote pseudo-labeling, we show that when the correct answer has
positive probability margin, the probability of pseudo-label failure decreases
exponentially with the number of rollouts.
Second, the restricted subspace must expose useful directions of the underlying
RL gradient.
We characterize this property through \emph{accessible gradient energy}, the
squared norm of the full gradient projected into the trainable subspace, and
show that the guaranteed local improvement is governed by this quantity rather
than directly by subspace dimensionality.

This prediction is reflected empirically.
Across nine single-layer bias subspaces with exactly the same number of
trainable parameters, accessible gradient energy varies by orders of magnitude
and is strongly associated with downstream trainability
(Spearman $\rho=0.917$, $p=0.0013$).
Thus, a tiny subspace can support effective adaptation when it captures useful
optimization directions, while another subspace of identical size may fail.
Together with the pseudo-label analysis, this suggests a simple view of
restricted test-time learning: success requires both a sufficiently reliable
self-generated signal and a sufficiently aligned optimization subspace.

We instantiate this regime using additive bias vectors in decoder MLP layers.
For each unlabeled problem, the model generates multiple rollouts whose majority
answer defines a pseudo-label reward; GRPO then updates only the bias vectors,
while all pretrained model parameters remain frozen.
After optimization, the learned biases form a small fixed intervention that is
applied directly during inference.
Beyond aggregate accuracy, we evaluate held-out transfer, compare trainable
subspaces under matched parameter budgets, analyze reward reliability and
subspace trainability, and audit multimodal improvements for task-specific
reward and evaluation confounds.

\textbf{Contributions.}
We study whether self-generated test-time learning can remain effective under an
extreme optimization bottleneck, with a frozen 7B-scale backbone and only
$\sim$100K trainable parameters. We find that substantial adaptation survives
this restriction across text, vision-language, and audio reasoning, and that the
learned intervention transfers to 4,500 verified-disjoint MATH problems without
further optimization. Under matched parameter budgets, different trainable
subspaces exhibit markedly different behavior, showing that parameter count
alone does not determine trainability. We further provide theoretical and
empirical evidence that restricted adaptation depends jointly on the reliability
of the self-generated reward and the optimization signal accessible within the
trainable subspace.
\section{Related Work}
\label{sec:related}

\paragraph{Self-supervised and test-time reinforcement learning.}
A growing line of work trains language models on their own generations without ground-truth labels.
Self-consistency \citep{wang2022selfconsistency} first showed that sampling many reasoning paths and marginalizing to the most frequent answer improves chain-of-thought accuracy at inference time; TTRL \citep{zuo2025ttrl} turns this majority-vote signal into a \emph{training} reward, updating a model's full parameter set against pseudo-labels computed on the unlabeled evaluation set itself ($\approx$211\% relative pass@1 improvement on AIME 2024 for Qwen2.5-Math-7B), with no human annotation.
Earlier bootstrapping approaches reach a related destination differently: STaR \citep{zelikman2022star} fine-tunes on model-generated rationales filtered by whether they reach a known-correct answer, ReST \citep{gulcehre2023rest} alternates between generating a training set from the current policy and fitting to it via offline RL, and self-rewarding language models \citep{yuan2024selfrewarding} use the model itself as an LLM-as-a-judge to build preference pairs for iterative DPO \citep{pmlr-v235-xiong24a}.
All of these, including TTRL \citep{zuo2025ttrl}, update every model parameter.
We adopt TTRL's majority-vote reward unchanged, optimized with group-relative policy optimization \citep[GRPO;][]{shao2024deepseekmath}, and ask a question this literature has not addressed: does the self-consistency signal remain effective when only a $\sim$100K-parameter bias vector is trainable, and does it generalize across different modalities.

\paragraph{Parameter-efficient fine-tuning and activation-space steering.}
A separate line of work asks how much of a model's behavior can be changed by training only a small fraction of its parameters, or none at all.
Adapter-based methods such as prefix-tuning \citep{li2021prefixtuning} and low-rank adaptation \citep[LoRA;][]{hu2022lora} insert or reparameterize a small number of trainable parameters into an otherwise frozen network (we compare directly against a LoRA parameterization of our own pipeline in \S\ref{sec:rq1}).
Activation-space steering goes further, adding a fixed vector to a model's residual stream at inference time with no gradient-based training at all \citep{turner2023activation}, or extracting that vector from the difference between contrastive positive and negative example activations \citep{rimsky2024steering}; representation engineering \citep{zou2023repe} frames this family as population-level manipulation of high-level concepts rather than individual neurons or circuits.
\citet{sinii2025steering} sit at the intersection of these two lines: rather than a training-free or difference-of-means vector, they \emph{train} one additive bias term per decoder layer with RL against a \emph{labeled, verifiable} reward (RLOO\citep{ahmadian2024back} on DeepScaleR) and match full RL fine-tuning at a fraction of a percent of the model's parameters.
This is our primary comparison point throughout the paper: we use the same bias-only mechanism and a similar reward-optimization setup (GRPO in place of RLOO \citep{ahmadian2024back}), but replace the labeled reward with TTRL's unlabeled majority-vote signal, and extend the comparison from text to vision and audio.

\paragraph{Reinforcement learning for multimodal reasoning.}
Post-training multimodal models with RL rather than supervised fine-tuning is active outside the label-free setting studied here.
Visual-RFT \citep{liu2025visualrft} applies verifiable, rule-based rewards (IoU for detection, exact-match for classification) with GRPO-style optimization, showing large gains in low-data regimes.
Closer to our setting, MM-UPT \citep{mmupt2025} explores unsupervised, self-rewarding post-training for multimodal LLM reasoning on MathVista-style tasks \citep{lu2024mathvista}, but fine-tunes the full model rather than an additive bias term.
Both update substantially more parameters than the $\sim$100K-parameter vectors studied here, and neither combines multimodal post-training with a fully label-free, test-time reward signal; our MathVista \citep{lu2024mathvista} and MMAU \citep{mmau2024} comparisons (\S\ref{sec:rq3}) are explored under our proposed method.

\section{Preliminaries and Problem Setup}
\label{sec:prelim}

\subsection{Problem Setup}
\label{sec:problemsetup}

We are given a pretrained model $\pi_\theta$ and an unlabeled reasoning dataset $\mathcal{D}=\{x_i\}_{i=1}^N$; no ground-truth answer is available for any $x_i$.
Rather than updating the full parameter set $\theta$, as standard RL fine-tuning would, our goal is to learn a small set of additive bias parameters $\phi$, one vector per targeted decoder layer, that steer the frozen model's behavior on $\mathcal{D}$; $\theta$ itself is never modified.
This design separates the two costs of adapting a pretrained model discussed in \S\ref{sec:introduction}: $\phi$ is orders of magnitude smaller than $\theta$, addressing parameter cost, and the reward that trains $\phi$ is derived entirely from the model's own rollouts on $\mathcal{D}$, addressing supervision cost.
The two components are independent, but combining them lets the same $\sim$100K-parameter approach apply unchanged across text, vision-language, and audio (\S\ref{sec:rq3}), where labeled, verifiable rewards are least available.
Two prior-work mechanisms supply this reward and this restriction; \S\ref{sec:pipeline} describes how we combine them.

\paragraph{TTRL.}
TTRL~\citep{zuo2025ttrl} turns majority-vote self-consistency~\citep{wang2022selfconsistency} into a \emph{training} reward.
For each problem $x_i$, $G$ rollouts $y_i^{(1)},\ldots,y_i^{(G)}\sim\pi_{\theta,\phi}(\cdot\mid x_i)$ are sampled from the current policy and the majority answer across the group, $\hat a_i = \mathrm{majority\_vote}\big(\{\mathrm{ans}(y_i^{(g)})\}_{g=1}^{G}\big)$, is taken as a pseudo-label; no ground truth is used anywhere in this loop, including for problems the model gets wrong, since an incorrect-but-self-consistent group still yields a pseudo-label.
Each rollout is then scored against $\hat a_i$,
\begin{equation}
r_i^{(g)} \;=\;
\begin{cases}
+1 & \text{if } \mathrm{ans}\big(y_i^{(g)}\big) = \hat a_i \\
-1 & \text{otherwise,}
\end{cases}
\label{eq:reward}
\end{equation}

\paragraph{GRPO.}
GRPO~\citep{shao2024deepseekmath} converts a group of per-rollout rewards into a group-normalized advantage, avoiding the need for a learned value function,
\begin{equation}
A_i^{(g)} \;=\; \frac{r_i^{(g)} - \mathrm{mean}\big(\{r_i^{(g')}\}_{g'=1}^{G}\big)} {\mathrm{std}\big(\{r_i^{(g')}\}_{g'=1}^{G}\big) + \epsilon},
\label{eq:advantage}
\end{equation}
applied identically at every token of the rollout.
We adopt TTRL's reward and GRPO's advantage unchanged as the reward source and optimizer for the pipeline in \S\ref{sec:pipeline}.

\paragraph{Bias-only steering.}
Bias-only steering~\citep{sinii2025steering} trains one additive bias term per decoder layer with RL against a \emph{labeled} reward (\S\ref{sec:related}).
At every targeted layer $l$, the bias vector $b_l\in\mathbb{R}^d$ is added directly to that layer's output; this modifies the layer's MLP block,
\begin{equation}
\mathrm{FFN}_l(h_l) \;=\; W_{\mathrm{down},l}\, \sigma\!\left(W_{\mathrm{up},l}\,h_l\right) \;+\; b_l,
\label{eq:steer}
\end{equation}
where $h_l$ is layer $l$'s input hidden state, $\sigma$ the MLP activation, and $W_{\mathrm{down},l}$, $W_{\mathrm{up},l}$ the frozen, pretrained down- and up-projection weight matrices; $b_l$ is the only quantity ever updated, and every other parameter, including all attention, embedding, and normalization weights, is fixed at its pretrained value.

\section{Method}
\label{sec:setup}

\subsection{Label-Free Bias-Only TTRL}
\label{sec:pipeline}
\label{sec:steering}

\begin{figure*}[t]
\centering
\includegraphics[width=0.94\linewidth]{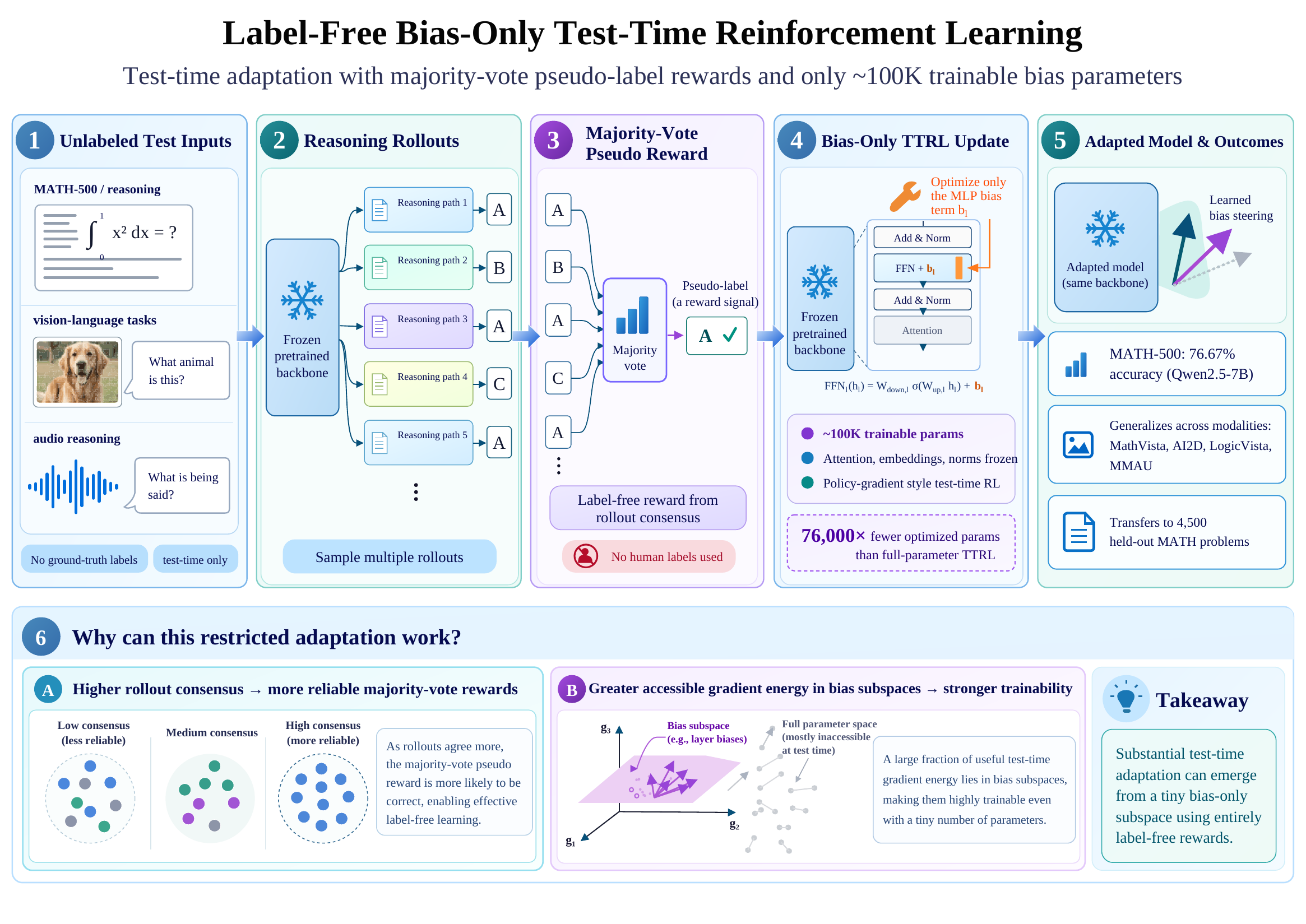}
\caption{Label-free bias-only TTRL. For unlabeled test inputs (1), the frozen backbone generates multiple rollouts (2) whose majority-vote answer serves as a pseudo-label reward (3); we train only a $\sim$100K-parameter set of bias vectors $\phi$ against this reward, with no labeled data and no update to the backbone $\theta$ (4, formalized as Algorithm~\ref{alg:ttrl}).
The identical procedure, changing only the prompt and reward/grading function, is validated on text (\S\ref{sec:rq1}) and transfers unchanged to vision-language and audio reasoning (5, \S\ref{sec:rq3}); panel 6 previews the two mechanisms analyzed in \S\ref{sec:theory}.}
\label{fig:overview}
\end{figure*}

Figure~\ref{fig:overview} (panels 2--4) gives the loop's visual overview; Algorithm~\ref{alg:ttrl} in Appendix~\ref{app:repro} states it formally. For each problem $x_i$, we sample $G$ rollouts and compute the pseudo-label reward (Eq.~\ref{eq:reward}) and GRPO advantage (Eq.~\ref{eq:advantage}) exactly as in \S\ref{sec:prelim}, restricted throughout to the bias-only subspace $\phi$: gradients computed from $A_i^{(g)}$ are back-propagated only into $\phi$, while $\theta$ receives no gradient and is held frozen.
This combination -- TTRL's label-free majority-vote reward, optimized with GRPO, restricted to the bias-only subspace $\phi$ -- is what the rest of this paper studies.

The learned bias vectors $b_l$ (Eq.~\ref{eq:steer}) are added to each targeted layer's output identically during training rollouts and evaluation; after training, $\phi$ is a single frozen artifact.
We use a simple additive offset rather than a multiplicative or low-rank update because it is the smallest possible per-layer intervention with a direct mechanistic reading ($b_l$ shifts the layer's output distribution by a constant vector regardless of $h_l$); it is also the mechanism used by \citet{sinii2025steering}, which keeps our label-free comparison isolated to the reward signal.

\subsection{Theoretical Grounding}
\label{sec:theory}

Two theoretical claims motivate this design, connecting the majority-vote reward and the restricted bias subspace (\S\ref{sec:pipeline}) to why the resulting optimization signal should be both reliable and usable.
We state simplified propositions that capture the qualitative mechanisms studied experimentally; full statements and proof sketches are provided in Appendix~\ref{app:theory-extra}.
\begin{proposition}[Pseudo-Label Reliability]
\label{prop:pseudo_label_reliability}
For an unlabeled input $x$, let $p(a\mid x)$ denote the model's
answer distribution over $K$ possible answers, and let $a^\star$
denote the correct answer. Let $\hat a_G$ be the majority-vote answer
obtained from $G$ i.i.d.\ rollouts, and define the answer margin as
\[
\Delta(x)
\triangleq
p(a^\star\mid x)
-
\max_{a\neq a^\star} p(a\mid x).
\]
If $\Delta(x)>0$, then the probability that the majority-vote
pseudo-label is incorrect is bounded by
\begin{equation}
\Pr\!\left[\hat a_G \neq a^\star\right]
\le
(K-1)
\exp\!\left(
-\frac{G\Delta(x)^2}{2}
\right).
\label{eq:prop1}
\end{equation}
\end{proposition}

Proposition~\ref{prop:pseudo_label_reliability} shows that, for
positive-margin problems, the pseudo-label error probability decreases
exponentially with the number of rollouts $G$. We test this qualitative
prediction empirically in \S\ref{sec:prop1} by measuring pseudo-label
accuracy as the number of rollouts increases.
\begin{theorem}[Restricted-Subspace Improvement Bound]
\label{thm:restricted_subspace}
Let $J(w)$ denote the RL objective and let
$g=\nabla J(w)$ be its gradient at $w$.
Consider adaptation restricted to a subspace $\mathcal{S}$, and let
$P_{\mathcal{S}}$ denote the orthogonal projection onto
$\mathcal{S}$. Define the accessible gradient energy as
\[
E_{\mathcal{S}}
\triangleq
\left\|P_{\mathcal{S}}g\right\|^2.
\]
Suppose that $J$ is locally $L$-smooth. For the restricted gradient
update
\[
w^+
=
w+\eta P_{\mathcal{S}}g,
\]
with $0<\eta\le 1/L$, the objective improvement is bounded below by
\begin{equation}
J(w^+)-J(w)
\ge
\frac{\eta}{2}
\left\|P_{\mathcal{S}}g\right\|^2
=
\frac{\eta}{2}E_{\mathcal{S}}.
\label{eq:prop2}
\end{equation}
\end{theorem}

Theorem~\ref{thm:restricted_subspace} shows that the guaranteed local
improvement is governed by the accessible gradient energy rather than
directly by the dimensionality of $\mathcal{S}$. Consequently, a small
subspace can still support substantial adaptation when it contains a
sufficiently aligned component of the full optimization gradient,
whereas equal-sized subspaces need not be equally trainable. We test
this qualitative prediction empirically in \S\ref{sec:rq2}.
\subsection{Implementation}
\label{sec:implementation}

Our primary configuration trains $b_l$ at all 28 MLP layers ($28\times3584{=}100{,}352$ parameters on Qwen2.5-7B/Qwen2.5-Math-7B); \S\ref{sec:rq2} separately ablates a single-layer sweep of the same pipeline, and \S\ref{sec:rq1} compares against a LoRA parameterization, with full detail on the reward-shaping, scale, and parameterization-stability ablations in Appendix~\ref{app:rq2-extra}. 
We additionally test a length-aware reward variant (ShorterBetter), which targets the shortest \emph{correct} response length within a rollout group as a dynamic target rather than a fixed-length penalty.
$\phi$ is optimized with AdamW \citep{loshchilov2017decoupled} at bfloat16 precision; gradients and optimizer state are kept only for $\phi$, so peak training memory is dominated by activations and rollout sampling rather than by optimizer-state duplication of $\theta$.
Learning rate, $G$, and schedule are swept per task as ordinary hyperparameters (\S\ref{sec:expsetup}), with the full configuration summarized in Appendix~\ref{app:repro}. 
Every run trains for a fixed budget of 200 steps.
Dataset descriptions, prompts, and evaluation protocol are explained in \S\ref{sec:expsetup}.

\section{Experimental Setup}
\label{sec:expsetup}

\textbf{Models.}
Text: Qwen2.5-7B (base) \citep{qwen25report}  and Qwen2.5-Math-7B \citep{yang2024qwen25math}.
Vision-language: Qwen2.5-VL-7B-Instruct \citep{qwen25vl2025}.
Audio: Qwen2.5-Omni-7B \citep{qwen2omni2025}.

\textbf{Tasks.}
MATH-500 (500 problem subset of \citealt{hendrycksmath2021}, no separate train split: the training and evaluation set are the same 500 problems, according to the TTRL paradigm \citep{zuo2025ttrl}); MathVista testmini \citep{lu2024mathvista}; AI2D-TEST \citep{kembhavi2016ai2d}, full 3088-problem test split; LogicVista \citep{xiao2024logicvista}, complete 448-row test split; MMAU-mini test-mini \citep{mmau2024}.

\textbf{Evaluation protocol.}
Unless otherwise noted, all of our own results use greedy decoding, pass@1, and a single generation per problem.
The comparison numbers cited from prior work follow each source's own published protocol.


\section{Results}

We evaluate label-free bias-only TTRL across text, vision-language, and audio reasoning.
Unless otherwise stated, bias-only TTRL results are reported as mean $\pm$ standard deviation over three independent training seeds.

\subsection{Main Results Across Reasoning Benchmarks}
\label{sec:rq1}
\label{sec:rq3}

Table~\ref{tab:main_results} summarizes the performance of label-free bias-only TTRL across text, vision-language, and audio reasoning.
The same $\sim$100K-parameter training procedure is used across all tasks and modalities -- Qwen2.5-Math-7B and Qwen2.5-7B on MATH-500, Qwen2.5-VL-7B-Instruct on AI2D, LogicVista, and MathVista, and Qwen2.5-Omni-7B on MMAU -- with only the task-specific prompt and reward/grading function adapted to each benchmark.
\newcommand{\cellstat}[3]{%
\shortstack{%
#1\\[-1pt]
{\tiny\color{blue!55!black}$\Delta$#2}\\[-1pt]
{\tiny\color{green!45!black}#3}}}

\newcommand{\cellstatbold}[3]{%
\shortstack{%
\textbf{#1}\\[-1pt]
{\tiny\color{blue!55!black}\textbf{$\Delta$#2}}\\[-1pt]
{\tiny\color{green!45!black}\textbf{#3}}}}

\newcommand{\cellstatneg}[3]{%
\shortstack{%
#1\\[-1pt]
{\tiny\color{blue!55!black}$\Delta$#2}\\[-1pt]
{\tiny\color{red!75!black}#3}}}

\begin{table*}[t]
\centering
\scriptsize
\setlength{\tabcolsep}{4.0pt}
\renewcommand{\arraystretch}{1.02}

\resizebox{\textwidth}{!}{%
\begin{tabular}{l|l|c|c|c|c|c|c|>{\columncolor{blue!10}}c}
\toprule

Model &
Benchmark &
Baseline &
\shortstack{Self-\\consistency} &
\shortstack{LoRA\\(exact-match)} &
LoRA $r{=}16$ &
LoRA $r{=}64$ &
FullFT &
\cellcolor{blue!10}\textbf{Bias-only TTRL (ours)}
\\

& & &
\multicolumn{6}{c}{
\tiny\itshape
score /
\textcolor{blue!55!black}{$\Delta$ vs.\ baseline} /
\textcolor{green!45!black}{relative gain}
}
\\

\cmidrule{1-9}

Trainable params &
&
0 &
0 &
100K &
5.0M &
20.2M &
$\sim$7.6B &
\cellcolor{blue!10}\textbf{100K}
\\

\midrule

Qwen2.5-Math-7B &
MATH-500 &
56.0 &
\cellstat{65.53 $\pm$ 1.27}{+9.53}{$\uparrow$17.0\%} &
\cellstatbold{78.80 $\pm$ 1.25}{+22.80}{$\uparrow$40.7\%} &
\cellstat{77.70 $\pm$ 0.10}{+21.70}{$\uparrow$38.8\%} &
\cellstat{75.73 $\pm$ 2.42}{+19.73}{$\uparrow$35.2\%} &
\cellstat{83.80 $\pm$ 0.28}{+27.80}{$\uparrow$49.6\%} &
\cellstatbold{79.50 $\pm$ 0.70}{+23.50}{$\uparrow$42.0\%}
\\

\midrule

Qwen2.5-7B &
MATH-500 &
45.9 &
\cellstat{73.40 $\pm$ 0.69}{+27.50}{$\uparrow$59.9\%} &
\cellstat{69.27 $\pm$ 2.81}{+23.37}{$\uparrow$50.9\%} &
\cellstat{73.93 $\pm$ 0.82}{+28.03}{$\uparrow$61.1\%} &
\cellstatbold{75.60 $\pm$ 0.60}{+29.70}{$\uparrow$64.7\%} &
\cellstat{77.00 $\pm$ 0.53}{+31.10}{$\uparrow$67.8\%} &
\cellstatbold{76.67 $\pm$ 1.57}{+30.77}{$\uparrow$67.0\%}
\\

\midrule

Qwen2.5-VL-7B-Instruct &
AI2D &
73.58 &
\cellstat{75.95 $\pm$ 0.45}{+2.37}{$\uparrow$3.2\%} &
\cellstat{75.26 $\pm$ 0.51}{+1.68}{$\uparrow$2.3\%} &
\cellstat{77.60 $\pm$ 0.80}{+4.02}{$\uparrow$5.5\%} &
\cellstatbold{79.15 $\pm$ 0.17}{+5.57}{$\uparrow$7.6\%} &
\cellstat{80.18 $\pm$ 0.28}{+6.60}{$\uparrow$9.0\%} &
\cellstatbold{78.29 $\pm$ 0.03}{+4.71}{$\uparrow$6.4\%}
\\

\midrule

Qwen2.5-VL-7B-Instruct &
LogicVista &
37.50 &
\cellstat{40.85 $\pm$ 0.89}{+3.35}{$\uparrow$8.9\%} &
\cellstatneg{35.86 $\pm$ 0.68}{-1.64}{$\downarrow$4.4\%} &
\cellstat{40.60 $\pm$ 1.60}{+3.10}{$\uparrow$8.3\%} &
\cellstatbold{40.70 $\pm$ 1.91}{+3.20}{$\uparrow$8.5\%} &
\cellstat{46.13 $\pm$ 0.38}{+8.63}{$\uparrow$23.0\%} &
\cellstatbold{43.38 $\pm$ 1.01}{+5.88}{$\uparrow$15.7\%}
\\

\midrule

Qwen2.5-VL-7B-Instruct &
MathVista &
61.60 &
\cellstat{64.73 $\pm$ 1.10}{+3.13}{$\uparrow$5.1\%} &
\cellstatbold{63.77 $\pm$ 0.25}{+2.17}{$\uparrow$3.5\%} &
\cellstat{63.50 $\pm$ 0.90}{+1.90}{$\uparrow$3.1\%} &
\cellstat{63.20 $\pm$ 0.88}{+1.60}{$\uparrow$2.6\%} &
\cellstat{62.57 $\pm$ 0.40}{+0.97}{$\uparrow$1.6\%} &
\cellstatbold{65.40 $\pm$ 1.77}{+3.80}{$\uparrow$6.2\%}
\\

\midrule

Qwen2.5-Omni-7B &
MMAU &
57.40 &
\cellstat{60.53 $\pm$ 0.47}{+3.13}{$\uparrow$5.5\%} &
\cellstat{58.07 $\pm$ 0.25}{+0.67}{$\uparrow$1.2\%} &
\cellstatbold{60.10 $\pm$ 1.00}{+2.70}{$\uparrow$4.7\%} &
\cellstat{59.57 $\pm$ 1.52}{+2.17}{$\uparrow$3.8\%} &
\cellstat{61.90 $\pm$ 1.64}{+4.50}{$\uparrow$7.8\%} &
\cellstatbold{60.87 $\pm$ 0.37}{+3.47}{$\uparrow$6.0\%}
\\

\bottomrule
\end{tabular}%
}

\caption{
Main results across text, vision-language, and audio reasoning.
Each trained cell reports the score (mean $\pm$ std over three seeds),
the absolute improvement over the corresponding untrained baseline
(\textcolor{blue!55!black}{$\Delta$}), and the relative gain
(\textcolor{green!45!black}{$\uparrow$}).
Red denotes degradation.
The best LoRA configuration for each benchmark is bolded separately,
while \textbf{Bias-only TTRL (ours)} is highlighted in blue.
Self-consistency performs majority voting over $G$ samples without parameter updates.
}
\label{tab:main_results}
\end{table*}
Bias-only TTRL improves over the corresponding untrained model on all six benchmarks, reaching $79.50\%$ on MATH-500 with Qwen2.5-Math-7B and $76.67\%$ with Qwen2.5-7B, with gains of $+3.8$ to $+5.9$ points on the vision-language benchmarks and $+3.5$ on audio (per-benchmark values in Table~\ref{tab:main_results}).
The benefit of label-free bias-only adaptation is therefore not restricted to language-only mathematical reasoning, but extends across substantially different reasoning modalities.

The LoRA comparisons in Table~\ref{tab:main_results} isolate the effect of the trainable parameterization while keeping the label-free TTRL procedure fixed.
To control for trainable capacity specifically, we additionally compare against LoRA (exact-match), an $r{=}1$ adapter restricted to \texttt{q\_proj} on 14 layers that trains exactly 100K parameters -- identical to bias-only TTRL.
Bias-only TTRL outperforms both LoRA (exact-match) and LoRA $r{=}16$ on all six benchmarks, despite LoRA $r{=}16$ training roughly $50\times$ more parameters; LoRA (exact-match) even falls below the untrained baseline on LogicVista ($35.86\%$ vs.\ $37.50\%$).
Bias-only TTRL also outperforms LoRA $r{=}64$ on five of six benchmarks, losing only on AI2D.
Only FullFT -- which trains the entire 7.6B-parameter model -- exceeds bias-only TTRL on more than one benchmark, doing so on five of six (all but MathVista, where bias-only TTRL is best overall, ahead of FullFT itself).
Bias-only adaptation is thus not uniformly superior to LoRA, but label-free test-time RL clearly does not require a large trainable space.

Table~\ref{tab:main_results} also reports a self-consistency baseline --- majority-vote over $G$ samples at inference time, with no training at all --- isolating how much of the gain is ``free'' inference-time voting rather than a learned improvement.
Self-consistency alone already recovers a substantial part of the gain on MATH-500 (e.g.\ $65.53\%\pm1.27$ vs.\ bias-only TTRL's $79.50\%\pm0.70$ on Qwen2.5-Math-7B), but bias-only TTRL still improves further over it on every benchmark, indicating that training on the majority-vote signal adds more than voting alone.

\subsection{Ablations with Labels and Parameterization}
\label{sec:comparison-prior}

To isolate the effect of the label source from the effect of the trainable parameterization, we run a full $2\times2$ ablation crossing label source (majority-vote pseudo-label vs.\ ground-truth labels) with trainable parameterization (bias-only vs.\ full fine-tuning) on MATH-500, holding the rest of the training recipe fixed within each parameterization.
Table~\ref{tab:mainmath} summarizes all four cells, each our own reproduction; Appendices~\ref{app:labelcompare-extra}--\ref{app:rq4-extra} additionally compare against an externally trained labeled vector.

\begin{table}[h]
\centering
\scriptsize
\setlength{\tabcolsep}{4.2pt}
\renewcommand{\arraystretch}{1.05}

\begin{tabular}{lcccc}
\toprule
Method
& \shortstack{Labeled\\reward?}
& \shortstack{Trainable\\params}
& Qwen2.5-7B
& Qwen2.5-Math-7B \\
\midrule

\rowcolor{blue!8}
\textbf{Bias-only, pseudo-label (ours)}
& \textbf{\ding{55}}
& \textbf{100K}
& \textbf{76.67 $\pm$ 1.57}
& \textbf{79.50 $\pm$ 0.70} \\

Bias-only, labeled
& \checkmark
& 100K
& 75.47 $\pm$ 1.21
& 77.27 $\pm$ 0.70 \\

\midrule

Full fine-tune, pseudo-label
& \ding{55}
& 7.6B
& 77.00 $\pm$ 0.53
& 83.80 $\pm$ 0.28 \\

Full fine-tune, labeled
& \checkmark
& 7.6B
& \textbf{87.93 $\pm$ 0.99}
& \textbf{85.13 $\pm$ 1.33} \\

\bottomrule
\end{tabular}

\caption{
$2\times2$ ablation on MATH-500 crossing reward supervision
(pseudo-label vs.\ ground truth) with parameterization
(bias-only vs.\ full fine-tuning).
All four settings are our own reproductions under a matched recipe within each
parameterization. Results are mean $\pm$ standard deviation over three seeds.
}
\label{tab:mainmath}
\end{table}
On Qwen2.5-7B, pseudo-label bias-only TTRL reaches $76.67\%\pm1.57$, slightly above its labeled counterpart ($75.47\%\pm1.21$, a $1.2$-point gap); full fine-tuning shows the opposite pattern, with labeled full fine-tuning ($87.93\%\pm0.99$) $10.9$ points above pseudo-label full fine-tuning ($77.00\%\pm0.53$).
On Qwen2.5-Math-7B the same pattern holds but is smaller for full fine-tuning: pseudo-label bias-only ($79.50\%\pm0.70$) again slightly exceeds labeled bias-only ($77.27\%\pm0.70$, a $2.2$-point gap), while labeled full fine-tuning ($85.13\%\pm1.33$) exceeds pseudo-label full fine-tuning ($83.80\%\pm0.28$) by only $1.3$ points.

This asymmetry is consistent with Qwen2.5-Math-7B's majority-vote pseudo-labels already being close to ground truth on math problems, while Qwen2.5-7B's pseudo-labels are noisier: removing labels costs full fine-tuning little on Math-7B but much more on the base model, while bias-only adaptation is not costed by removing labels in either model.
Under pseudo-labels, bias-only adaptation stays within $0.3$--$4.3$ points of full fine-tuning while updating approximately $76{,}000\times$ fewer parameters.

Training compute differs sharply between parameterizations: bias-only training uses a single GH200 GPU ($4.45$--$4.95$ GPU-hours/run), while full fine-tuning needs four GH200 GPUs under FSDP ($79.24$--$82.12$ GPU-hours/run), a $16.0$--$18.5\times$ reduction in GPU-hours on top of the parameter reduction.
Pseudo-label and labeled variants share an identical training recipe per parameterization; the reward source only changes a negligible-cost computation over already-generated rollouts, so their compute is expected to match.
\subsection{Generalization to Held-Out MATH Problems}
\label{sec:rq1-transfer}

To test whether the learned bias vectors generalize beyond the problems used for
test-time optimization, we freeze the step-200 checkpoints trained on MATH-500
and evaluate them, without further training, on 4,500 verified-disjoint MATH
problems.

\begin{table}[h]
\centering
\scriptsize
\setlength{\tabcolsep}{7pt}
\renewcommand{\arraystretch}{1.08}
\begin{tabular}{lccc}
\toprule
Model & Unadapted & Transferred bias & $\Delta$ \\
\midrule
Qwen2.5-7B
& 46.3
& \textbf{70.9}
& \textcolor{green!45!black}{\textbf{+24.6}} \\

Qwen2.5-Math-7B
& 52.5
& \textbf{75.4}
& \textcolor{green!45!black}{\textbf{+22.9}} \\
\bottomrule
\end{tabular}
\caption{
Transfer to 4,500 verified-disjoint MATH problems.
Bias vectors are trained on MATH-500 and applied without further optimization.
Values are accuracy (\%); $\Delta$ is the absolute gain in percentage points.
}
\label{tab:heldout_transfer}
\end{table}

The transferred bias vectors improve Qwen2.5-7B by $24.6$ points and
Qwen2.5-Math-7B by $22.9$ points (Table~\ref{tab:heldout_transfer}). Since none of the held-out problems is used
during optimization, the gains indicate that the learned intervention transfers
beyond the fixed MATH-500 problem set.
\subsection{Majority-Vote Pseudo-Label Accuracy}
\label{sec:prop1}

Proposition~1 (\S\ref{sec:theory}) predicts that majority-vote pseudo-label failure decays exponentially in the number of rollouts $G$ for positive-margin problems.

As shown in Figure~\ref{fig:theory_validation}(a,b), we test this prediction on an $N{=}64$ batch of MATH-500 problems: $G_{\max}{=}64$ rollouts are generated once per problem, majority vote is computed as a prefix for each $G$, and the empirical consensus proxy $m=(N_1-N_2)/64$ is recorded alongside accuracy.


The ranking of pseudo-label accuracy tracks both rollout count and consensus strength, giving empirical support for the qualitative prediction of Proposition~1.

\subsection{Restricted-Subspace Trainability}
\label{sec:rq2}

We next examine why restricting optimization to approximately 100K bias parameters can still produce substantial adaptation.
Section~\ref{sec:theory} predicts that equal-sized bias subspaces need not be equally trainable: what matters is the amount of useful RL gradient lying within the accessible subspace.

As shown in Figure~\ref{fig:theory_validation}(c), we test this prediction using nine single-layer bias subspaces spanning early, middle, and late layers of the network, each containing exactly 3,584 parameters.
For each layer, we compute the accessible gradient energy $E_l=\|P_l g\|^2$ from a shared backward pass on a fixed 16-problem batch and compare it with the final-step accuracy obtained after training that single-layer subspace.
\begin{figure*}[h]
    \centering
    \includegraphics[width=\textwidth]{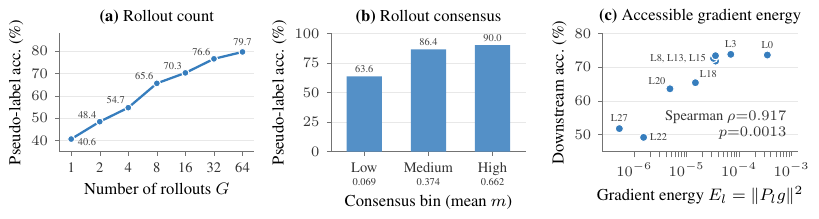}
    \caption{
    Empirical support for Proposition~1 and Theorem~1.
    \textbf{(a)} Majority-vote pseudo-label accuracy increases monotonically with the number of rollouts $G$.
    \textbf{(b)} Higher rollout consensus corresponds to higher pseudo-label accuracy.
    \textbf{(c)} Accessible gradient energy is strongly associated with downstream training accuracy across nine single-layer bias subspaces
    (Spearman $\rho=0.917$, exact two-tailed permutation test, $p=0.0013$, $n=9$).
    Together, these results are consistent with the qualitative predictions that stronger rollout agreement yields more reliable pseudo-labels and that restricted subspaces with greater accessible gradient energy are more trainable.
    }
    \label{fig:theory_validation}
\end{figure*}

The ranking of accessible gradient energy is consistent with downstream trainability, giving a Spearman correlation of $\rho=0.917$ between accessible gradient energy and training accuracy across all nine layers (exact two-tailed permutation test, $p=0.0013$).
Thus, equal parameter counts do not imply equal trainability: layers whose bias directions contain more accessible optimization signal are more effective.
This provides empirical support for the qualitative prediction of Theorem~1.

\subsection{Analysis of Multimodal Improvements}
\label{sec:rq3-audit}

Because test-time learning can potentially exploit task-specific properties of the reward or evaluation procedure, we audit the multimodal gains for possible confounds.
We manually inspect improved and degraded examples and then perform exhaustive checks for the patterns identified by the manual analysis (full methodology in Appendix~\ref{app:qual-extra}).
On AI2D, training substantially reduces the official parser's failure-to-extract rate, indicating that part of the improvement comes from better output-format compliance rather than only changes in answer content.
For LogicVista, scored with the same chain-of-thought prompt and grader for baseline and steered models, an exhaustive check of all 448 problems finds 77 flipped from wrong to correct and 47 from correct to wrong, with no single dominant confound.
For MathVista, the steered model answers exactly \texttt{0} on all 49 age-gap questions, a narrow case that contributes only one net problem to the gain.
For MMAU, a single phoneme-counting template (15.3\% of the set) accounts for roughly two-thirds of the improvement: $+15.1$ points on that template versus $+1.4$ on the remaining 847 problems, which is within seed variation (Table~\ref{tab:mmauconfound}); 

Overall, the audits find no single dominant confound explaining the improvements across the multimodal benchmarks, while identifying narrow task-specific cases where the label-free reward can influence the learned behavior; per-case examples are in Table~\ref{tab:qualexamples}.

\FloatBarrier
\section{Conclusion}

We introduced label-free bias-only TTRL, which optimizes only $\sim$100K bias parameters using majority-vote pseudo-label rewards while keeping the pretrained backbone frozen.
On MATH-500, it reaches 79.50\% with Qwen2.5-Math-7B and 76.67\% with Qwen2.5-7B, slightly outperforming matched labeled bias-only adaptation.
The same procedure improves vision-language and audio reasoning and transfers to 4,500 held-out MATH problems.
Our analyses link successful adaptation to reliable majority-vote rewards and accessible gradient energy, showing that effective test-time adaptation can emerge from a tiny trainable subspace without ground-truth labels or full-model updates.
%

\section*{Acknowledgments}

\begin{wrapfigure}{l}{0.24\textwidth}
\vspace{-1.2em}
\centering
\includegraphics[width=0.85\linewidth]{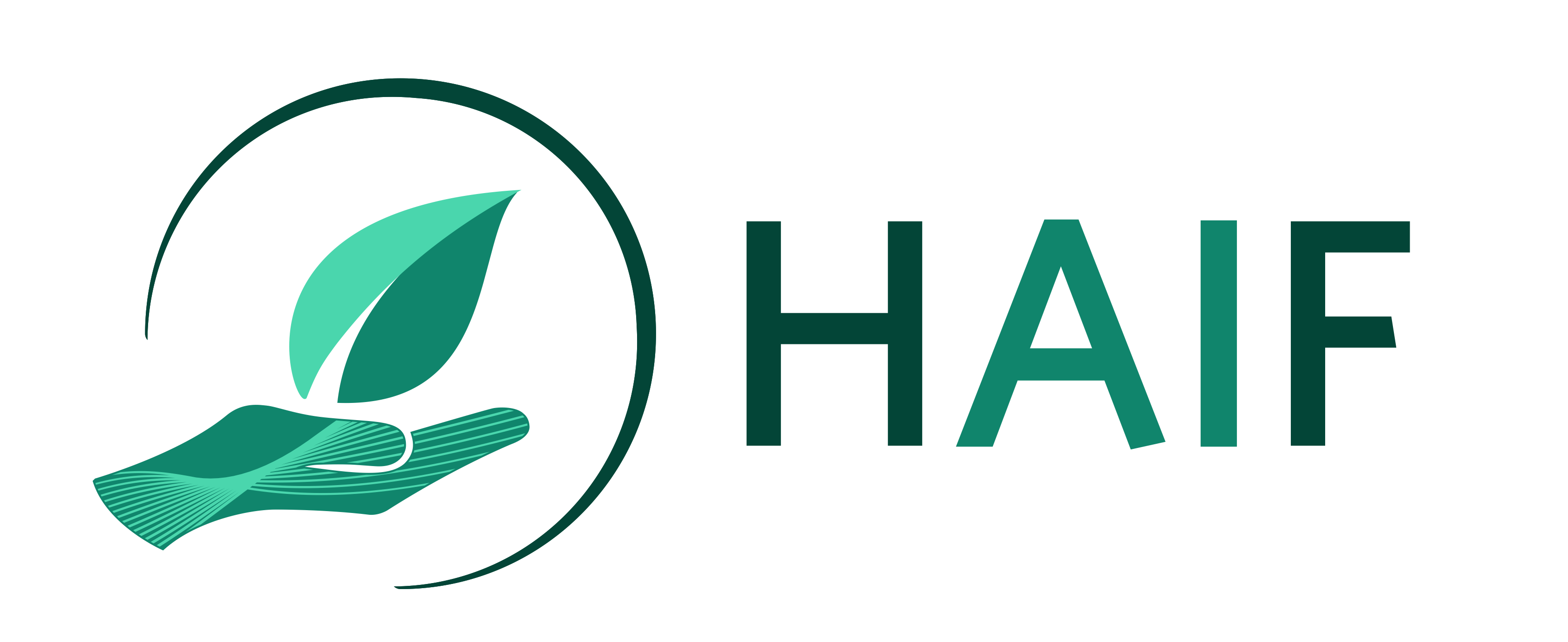}\\[0.6em]
\includegraphics[width=\linewidth]{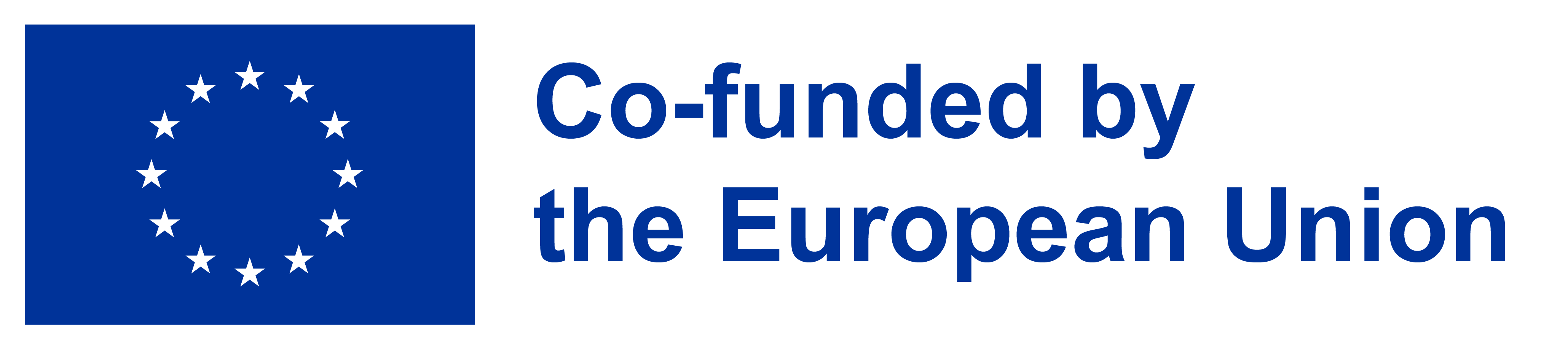}
\vspace{-1.5em}
\end{wrapfigure}
The authors wish to acknowledge CSC -- IT Center for Science, Finland, for computational resources, including access to the Roihu and LUMI supercomputers.
This work was supported by HAIF, co-funded by the European Union's Horizon Europe research and innovation programme's Marie Sk\l{}odowska-Curie Action.
Co-funded by the European Union. 
Shaoxiong Ji gratefully acknowledges the support of Foundation PS through the PS Fellowship.

\section*{Statement on the Use of Generative AI}

In preparing this manuscript, we used a generative AI assistant for paper-preparation
support: \LaTeX{} formatting and table construction, drawing figures, condensing and
copy-editing prose that we had written, and cross-checking numerical consistency
between the tables, the figures, and the surrounding text.
Every number reported in this paper comes from our own experimental runs; the
assistant was given those numbers in order to typeset and cross-check them, and was
not used to produce, estimate, or extrapolate any result.

We did not use generative AI tools to generate synthetic data or to write any part of
the training or evaluation pipeline whose outputs are reported here.
The theoretical statements and proof sketches in \S\ref{sec:theory} and
Appendix~\ref{app:theory-proofs} are our own.

All AI-assisted edits were checked by the authors against the underlying experimental
logs before inclusion, and the final manuscript was read in full by the authors, who
take responsibility for its entire contents.

\section*{Ethics Statement}

This work fine-tunes publicly released model checkpoints on publicly released reasoning
benchmarks (MATH-500, AI2D, LogicVista, MathVista, and MMAU-mini) and introduces no new
data collection.
No human subjects were involved and no personally identifying information was used.
Because our method adapts a model using rewards derived from the model's own outputs
rather than from verified labels, it can in principle reinforce a model's existing
errors rather than correct them; \S\ref{sec:rq3-audit} and Appendix~\ref{app:qual-extra}
report the cases in which we observed this, and we consider auditing for such effects a
necessary part of deploying label-free test-time adaptation rather than an optional
check.

\section*{Reproducibility Statement}

The full training and evaluation configuration for every reported result, trainable
parameter set, optimizer, rollout count, learning rate and schedule, step budget,
decoding settings, and checkpoint-selection rule is given in
Appendix~\ref{app:repro}.
All of our own results are reported as mean $\pm$ standard deviation over three
independent training seeds under a fixed 200-step budget, with the final checkpoint
reported in every case.
The context lengths and decoding protocols underlying each row of
Table~\ref{tab:main_results} are listed in Appendix~\ref{app:repro}, and the additional
reward-shaping, scale, and parameterization ablations are specified in
Appendix~\ref{app:rq2-extra}.

\bibliography{references}

\appendix
\input{appendix}

\end{document}

%% file: appendix.tex
\section{Reproducibility and Configuration Details}
\label{app:repro}

Algorithm~\ref{alg:ttrl} states the training loop described in \S\ref{sec:pipeline} formally; Table~\ref{tab:repro} then gives the full configuration behind every reported result.

\begin{algorithm}[h]
\centering
\begin{tcolorbox}[colback=gray!3,colframe=black!55,boxrule=0.4pt,
left=5pt,right=5pt,top=2pt,bottom=2pt,width=0.82\linewidth]
\scriptsize
\textbf{Require:} frozen backbone $\pi_\theta$, unlabeled dataset $\mathcal{D}$, rollouts $G$, learning rate $\eta$, steps $T$

\vspace{1pt}
\begin{tabbing}
\hspace{1.0em}\=\hspace{1.0em}\=\hspace{1.0em}\=\hspace{1.0em}\=\kill 1:\ \>$\phi \gets \mathbf{0}$ \\
2:\ \>\textbf{for} $t = 1, \ldots, T$ \textbf{do} \\
3:\ \>\>sample a batch $\mathcal{B} \subset \mathcal{D}$ \\
4:\ \>\>\textbf{for} each $x_i \in \mathcal{B}$ \textbf{do} \\
5:\ \>\>\>sample $y_i^{(1)}, \ldots, y_i^{(G)} \sim \pi_{\theta,\phi}(\cdot \mid x_i)$ \\
\>\>\>\textit{(bias per Eq.~\ref{eq:steer})} \\
6:\ \>\>\>$\hat a_i \gets \mathrm{majority\_vote}\big(\{\mathrm{ans} (y_i^{(g)})\}_{g=1}^{G}\big)$ \\
7:\ \>\>\>\textbf{for} $g = 1, \ldots, G$ \textbf{do} \\
8:\ \>\>\>\>$r_i^{(g)} \gets +1$ if $\mathrm{ans}(y_i^{(g)}) = \hat a_i$ else $-1$ \\
\>\>\>\>\textit{(Eq.~\ref{eq:reward})} \\
9:\ \>\>\>\textbf{end for} \\
10:\>\>\>$A_i^{(g)} \gets$ group-normalized advantage of $\{r_i^{(g)}\}_{g=1}^{G}$ \\
\>\>\>\textit{(Eq.~\ref{eq:advantage})} \\
11:\>\>\textbf{end for} \\
12:\>\>$\phi \gets \phi + \eta\,\nabla_\phi \sum_i \sum_g A_i^{(g)} \log \pi_{\theta,\phi}\big(y_i^{(g)} \mid x_i\big)$ \\
\>\>\textit{($\theta$ fixed)} \\
13:\ \>\textbf{end for} \\
14:\ \>\textbf{return} $\phi$
\end{tabbing}
\end{tcolorbox}
\caption{Label-Free Bias-Only TTRL}
\label{alg:ttrl}
\end{algorithm}

\begin{table*}[h]
\centering
\footnotesize
\resizebox{\textwidth}{!}{%
\begin{tabular}{@{}p{1.9cm}p{4.7cm}p{4.7cm}p{4.7cm}@{}}
\toprule
& Text & Vision-Language & Audio \\
\midrule
Base model & Qwen2.5-7B / Math-7B & Qwen2.5-VL-7B-Instruct & Qwen2.5-Omni-7B \\
Trainable params & 28 $\times$ 3584 $=$ 100{,}352 & same & same (text decoder only) \\
Param type & \texttt{down\_proj.bias}, all 28 MLP layers & same & same \\
Algorithm & GRPO, $G{=}32$--$64$ rollouts, majority vote & same & same \\
LR / schedule (best) & 1e-3, cosine, 200 steps & 1e-3, delayed-cosine (60 flat + 140 decay) & 1e-3, constant, 200 steps \\
Train data & MATH-500 (no separate train split) & MathVista testmini, AI2D test, LogicVista test & MMAU-mini test-mini \\
Labels & none, majority vote only & same & same \\
Eval & greedy, pass@1 & greedy, letter extraction, $n{=}1000$ / $3088$ / $448$ & greedy, letter extraction, $n{=}1000$ \\
Ckpt selection & last checkpoint (step 200) & same & same \\
\bottomrule
\end{tabular}}
\caption{Reproducibility and hyperparameter configuration summary.
Learning rate schedule was swept per task (constant vs.\ a cosine-decay variant) as an ordinary hyperparameter; the row above records the schedule used for training, rather than treating the schedule comparison itself as a reported finding.
Every run trains for a fixed budget of 200 steps, and ``Ckpt selection: last checkpoint (step 200)'' means the reported number for each headline result is the final checkpoint.}
\label{tab:repro}
\end{table*}

\section{Additional Ablations: Reward Shaping and Scale}
\label{app:rq2-extra}

This appendix expands the reward-shaping and scale ablations referenced in the main paper's \S\ref{sec:implementation} (Implementation) with the full detail behind each finding.

\paragraph{Reward shaping.}
ShorterBetter reaches a slightly higher peak than plain binary GRPO reward (75.2\% vs.\ 74.2\%, Qwen2.5-7B) while cutting average response length by $\approx$30\% (387$\to$253 words over the course of training), shorter and more accurate simultaneously, not a length/accuracy trade-off.

\paragraph{Model scale and math-specialization.}
Qwen2.5-Math-7B (78.0\% peak) outperforms Qwen2.5-7B base (74.6\%) by roughly 3pp, consistent with stronger math pretraining giving more headroom.
At 1.5B scale, both Qwen2.5-1.5B (+54pp) and Qwen2.5-Math-1.5B (+37pp) show large absolute gains from a near-zero base, though both plateau below their 7B counterparts.

\paragraph{Parameterization stability: a bias-free ReFT baseline \citep{wu2024reft}.}
Table~\ref{tab:main_results}'s LoRA (exact-match) baseline shows that trainable-parameter count alone does not determine label-free TTRL performance; we additionally test whether it determines \emph{training stability}, using a second parameter-matched but architecturally distinct intervention.
Following the low-rank interchange-intervention parameterization of ReFT-style methods, 
but dropping the bias term and orthogonality constraint (a ``NoDiReFT''-style variant), we inject $h' = h + W_2^\top(W_1 h)$ -- two rank-1 linear maps, no additive bias -- at the output of 14 evenly spaced decoder layers (0, 2, 4, \ldots, 26) via a forward hook, giving $14 \times 2 \times 3584 = 100{,}352$ trainable parameters, identical to bias-only steering and to LoRA (exact-match).
Training otherwise follows the same label-free majority-vote GRPO recipe over three seeds, with learning rate, step budget, and $G$ matched per task to bias-only's own configuration (MATH-500: lr$=5\times10^{-4}$, 200 steps, $G{=}64$; AI2D/LogicVista/MathVista/MMAU: lr$=10^{-3}$, 200 steps, $G{=}32$).

\begin{table}[h]
\centering
\scriptsize \setlength{\tabcolsep}{6pt}
\begin{tabular}{lccc}
\toprule
Benchmark & Baseline & Final step (reported) & Best step reached \\
\midrule
MATH-500 (Math-7B) & 56.0 & $61.60 \pm 19.28$ & $71.87 \pm 3.10$ \\
MATH-500 (Qwen2.5-7B) & 45.9 & $67.00 \pm 12.76$ & $73.93 \pm 1.21$ \\
AI2D & 73.58 & $60.31 \pm 29.74$ & $78.47 \pm 0.36$ \\
LogicVista & 37.50 & $23.29 \pm 20.63$ & $39.51 \pm 0.39$ \\
MathVista & 61.60 & $56.60 \pm 16.39$ & $66.00 \pm 2.62$ \\
MMAU & 57.40 & $53.53 \pm 1.05$ & $55.13 \pm 2.01$ \\
\bottomrule
\end{tabular}
\caption{Bias-free ReFT at the same 100K-parameter budget as bias-only steering; mean $\pm$ std over three seeds. Reported numbers are the final step, as everywhere else in this paper; the best step reached is shown to separate peak performance from end-of-training performance.}
\label{tab:reft}
\end{table}

The gap between the two columns of Table~\ref{tab:reft} is the finding: at an identical parameter budget, ReFT reaches a competitive peak but does not hold it.
Runs improve early, then degrade before the step budget ends, so the reported final-step numbers fall by up to 18 points below the best step reached and carry standard deviations of 12--30 points on five of six benchmarks -- behaviour consistent with a bilinear, KL-free intervention that has no mechanism pulling the policy back toward the reference model once it drifts.
Bias-only and LoRA, trained on the same budget and recipe, show no comparable degradation (Table~\ref{tab:main_results}).
Even at its peak, ReFT stays below the untrained baseline on MMAU and only marginally exceeds it on LogicVista.
This supports the design choice motivated in \S\ref{sec:pipeline}: at a matched trainable-parameter budget, not every restricted parameterization is equally trainable under a label-free, KL-free RL objective.

\section{Qualitative Audit: Full Detail}
\label{app:qual-extra}

This appendix expands the main paper's \S\ref{sec:rq3-audit} (Analysis of Multimodal Improvements) summary with full audit methodology, per-case examples, and the two underlying breakdown tables.

\paragraph{Audit methodology.}
We manually audited 10 randomly-sampled \emph{improved} examples (baseline wrong $\to$ steered correct, seed 42) and 10 randomly-sampled \emph{error} examples (steered still wrong, seed 7) for MathVista and MMAU, cross-checked against source images (MathVista) or textual self-consistency; for LogicVista we exhaustively tallied every helped/hurt/hard case over the full $n{=}448$ set and manually reviewed 7 \emph{improved} and 2 \emph{hurt} cases.
Beyond these samples, we ran full-dataset quantitative checks for every pattern they surfaced; Table~\ref{tab:qualexamples} shows one improved and one error case per benchmark.

\begin{table*}[h]
\centering
\scriptsize \setlength{\tabcolsep}{3pt} {\renewcommand{\arraystretch}{0.85}
\begin{tabular}{p{0.05\linewidth}p{0.28\linewidth}p{0.30\linewidth}p{0.07\linewidth}p{0.17\linewidth}}
\toprule
ID & Question (abridged) & Baseline $\to$ Steered & Gold & Verdict \\
\midrule
\multicolumn{5}{l}{\textbf{MathVista}} \\
\multicolumn{5}{l}{\emph{Improved (baseline wrong $\to$ steered correct)}} \\
320 & Right-triangle midpoints; find $DE$ ($AB{=}13$, $AC{=}5$, $\angle C{=}90^\circ$) & $DE{=}AB/2{=}6.5$ (wrong theorem) $\to$ $BC{=}\sqrt{13^2{-}5^2}{=}12$, $DE{=}6$ & 6 & Fixed; correct midsegment theorem \\
\midrule
\multicolumn{5}{l}{\emph{Error (steered still wrong)}} \\
210 & Age gap between two people in a photo & wrong numeric guess $\to$ answers ``0'' & nonzero & One of 49 age-gap questions; see mode-collapse analysis below \\
\midrule
\multicolumn{5}{l}{\textbf{LogicVista}} \\
\multicolumn{5}{l}{\emph{Improved (baseline wrong $\to$ steered correct)}} \\
v1\_203 & Compute a 2023 value after a 26\% production increase from 2022 (MCQ) & arithmetic left incomplete mid-calculation, D (wrong) $\to$ finished the chain ($20\text{K}\times0.26{=}5.2\text{K}$, $20\text{K}{+}5.2\text{K}{=}25.2\text{K}$), C (correct) & C & Fixed; completed the arithmetic chain baseline abandoned \\
\midrule
\multicolumn{5}{l}{\emph{Hurt (baseline correct $\to$ steered wrong)}} \\
v1\_0 & Which choice completes a row/column shape pattern (MCQ) & correctly tracked the row/column pattern $\to$ C (correct) $\to$ re-opened the analysis, fixated on irrelevant dot positions $\to$ D (wrong) & C & Over-correction: doubted its own already-correct reasoning \\
\midrule
\multicolumn{5}{l}{\textbf{MMAU}} \\
\multicolumn{5}{l}{\emph{Improved (baseline wrong $\to$ steered correct)}} \\
158 & Count words with $\geq$1 (un)stressed phoneme (153/1000 items are this template) & repetition-loop failure, empty prediction $\to$ ``five'' from 2 cited words & five & Confound: template alone drives $\sim$66\% of net gain \\
\midrule
\multicolumn{5}{l}{\emph{Error (steered still wrong, incl.\ regressions)}} \\
482 & Timestamp for a specific named chord & wrong answer, specific wording $\to$ near-identical wrong wording & -- & Persistent error; suggests convergence on a shared wrong pattern \\
\bottomrule
\end{tabular}}
\caption{MathVista, LogicVista, and MMAU: 1 improved $+$ 1 error/regression case each (MathVista rows drawn from 20 sampled cases; audio itself not directly verifiable for MMAU, audited for textual self-consistency and grading-code correctness).
LogicVista rows are drawn from 9 manually-reviewed cases (7 of the 77 helped cases, 2 of the 47 hurt cases); the remaining reviewed helped cases span shape/pattern discrimination, sequence completion, spatial reasoning, and simple physical intuition, and the other reviewed hurt case similarly shows symbol misidentification breaking otherwise-correct logic.
Images cross-checked against stated reasoning where applicable.}
\label{tab:qualexamples}
\end{table*}

\paragraph{MathVista and LogicVista: gains are largely genuine, with one narrow exception on MathVista.}
Both are vision-language MCQ/short-answer visual-reasoning tasks trained with the identical recipe.
On LogicVista, we did not find a single dominant confound: checked exhaustively across the full $n{=}448$ set, steering flips 77 baseline-wrong cases to correct and 47 baseline-correct cases to wrong (net $+$30 cases $\approx$ $+$6.7pp on this single audited run; the three-seed mean in Table~\ref{tab:main_results} is $+5.88$), with the remaining 203 cases wrong under both baseline and steered.
Accuracy rises monotonically across all four audited checkpoints (37.5\%$\to$38.6\%$\to$41.7\%$\to$42.6\%$\to$44.2\% at steps 0/20/80/160/200) with parse-failure falling alongside it (2.2\%$\to$0.7\% at step 160, ticking back up slightly to 1.1\% at step 200), the same benign pattern as AI2D's parser-failure reduction -- across all three vision-language benchmarks, the large majority of the reported gain reflects the model genuinely answering more problems correctly.
On MathVista, our exhaustive audit (not sampled) surfaced one narrow edge case worth flagging for transparency rather than a broad problem: the 1000-problem set contains 49 ``age gap between two people in a photo'' questions, and the steered model answers \emph{exactly} \texttt{0} on all 49 regardless of the true gap, correct only on the single case where gold happens to be 0 (reached via a hallucinated ``they are the same person'' claim we independently falsified against the source photo); baseline is also 0/49 correct on this template, mostly via honest ``cannot be determined'' refusals.
On this single audited run (net $+$6.2pp; three-seed mean $+3.80$ in Table~\ref{tab:main_results}), the template's net contribution is $+$1 problem out of 62 net-improved problems -- too small to meaningfully drive the headline number, but a useful illustration that majority-vote pseudo-labels can occasionally converge on a fixed answer rather than genuine reasoning on a narrow template.
It does not extend to any of the other 951 MathVista problems or to LogicVista.

\paragraph{MMAU: the phoneme-counting confound.}
The question template ``count the number of words that contain at least one [un]stressed phoneme'' appears 153/1000 times (15.3\%) in the full set.
Checked exhaustively (Table~\ref{tab:mmauconfound}):

\begin{table}[h]
\centering
\scriptsize \setlength{\tabcolsep}{3pt} {\renewcommand{\arraystretch}{0.85}
\begin{tabular}{lcccc}
\toprule
& $n$ & Baseline & Steered & $\Delta$ \\
\midrule
Phoneme subset & 153 & 40.5\% & 55.6\% & \posd{$+$15.1pp} \\
Everything else & 847 & 60.4\% & 61.9\% & \posd{$+$1.4pp} \\
\textbf{Full set} & \textbf{1000} & \textbf{57.4\%} & \textbf{60.9\%} & \posd{\textbf{+3.5pp}} \\
\bottomrule
\end{tabular}}
\caption{MMAU-mini, phoneme-counting subset vs.\ the rest of the set.}
\label{tab:mmauconfound}
\end{table}

Roughly two-thirds (23/35) of the entire reported MMAU improvement traces to this one 15\%-of-dataset template, on which neither model's stated reasoning logically derives its own stated answer: e.g.\ one response lists three word-mentions (``director,'' ``he's,'' ``he's'' again) as containing the target phoneme, then answers ``five.''
Precise phoneme-stress counting from raw audio is not a capability either model plausibly has; the pattern resembles the steered model learning, via majority-vote RL, to avoid baseline's ``count every word'' heuristic in favor of smaller numbers that correlate better with this dataset's answers, without evidence of genuine phonetic analysis.
An answer-letter distribution check rules out naive majority-class exploitation: gold answers skew toward \texttt{A} (37.7\%), but predicted-\texttt{A} rate \emph{falls} after steering (34.6\%$\to$31.1\%) while predicted-\texttt{B} rises (23.2\%$\to$30.7\%).
Regressions occur on 79/1000 MMAU problems vs.\ 56/1000 on MathVista; several show near-identical wording between baseline and steered outputs (id 482 above), suggesting convergence on a shared wrong pattern we did not isolate to a single template.

\paragraph{How the MMAU headline number should be read.}
We want to be explicit about what this does and does not leave standing, since the main paper reports MMAU as one of six improved benchmarks.
The $+3.5$pp headline gain on MMAU-mini is not uniformly distributed across the benchmark: it is concentrated on a single template covering $15.3\%$ of the set, and once that template is excluded the remaining $847$ problems improve by only $+1.4$pp.
A gain of that size is within the seed-to-seed variation we observe elsewhere in this paper, so we do not claim a demonstrated improvement on the non-template portion of MMAU.
The honest reading is that our audio result rests substantially on one question template whose gain we cannot attribute to genuine phonetic reasoning, and that MMAU is the weakest of our six benchmark results for this reason.
We report it because excluding a benchmark after seeing its audit would be a worse practice than reporting it with this caveat attached, and because the effect is diagnosable only through exactly this kind of per-template audit, which we would encourage for label-free test-time adaptation generally, since the reward provides no external signal that would flag such a shortcut during training.
The MathVista confound examined above is an order of magnitude smaller (one net problem out of 62) and does not carry the same weight.

\section{Labeled vs.\ Label-Free Vector Comparison: Full Detail}
\label{app:labelcompare-extra}

This appendix reports a full comparison between the learned label-free bias vector and a labeled bias-steering vector, examining whether label-free training produces a degenerate intervention or a task-specific change in model behavior.

Both compared vectors are all-layer \texttt{down\_proj.bias} vectors on Qwen2.5-7B; the label-free vector is trained with GRPO on MATH-500 itself, while the labeled comparison vector is Sinii \citep{sinii2025steering}-reproduced, trained with RLOO on 40K labeled DeepScaleR problems and evaluated here out-of-distribution, so the reward source, optimizer, and training data all differ, not only the label source.
Since the two checkpoints' originally-reported numbers used different context lengths, we re-evaluated both under one matched harness (same prompts, greedy decoding, context budget) on the same 500 MATH-500 problems: baseline 46.8\% (matching Sinii-repro's own reported step-0 exactly, confirming harness parity), label-free 72.0\%, labeled 75.6\%.
We additionally verified the baseline-vs-label-free delta's significance without relying on multi-seed training: a paired bootstrap (10,000 resamples) over this same 500-problem eval gives a 95\% CI of $[+20.6,+29.8]$pp for the accuracy delta, and McNemar's exact test on the 166 discordant predictions gives $p<10^{-20}$.

\paragraph{Relation to the label-source ablation in the main paper.}
This appendix and \S\ref{sec:comparison-prior} answer two different questions, and their labeled-vs-label-free comparisons therefore point in different directions; we state the distinction here so that neither result is read as contradicting the other.
The main-paper ablation (Table~\ref{tab:mainmath}) is a controlled comparison in which \emph{only} the reward source changes: both arms are our own runs, on the same MATH-500 problems, with the same optimizer, rollout count, learning rate, and step budget.
Under that control, removing ground-truth labels costs nothing at this parameter budget, and the pseudo-label arm is $1.2$--$2.2$ points ahead.
The comparison in this appendix is not controlled in that sense: the labeled vector is a reproduction of a different method, trained with a different optimizer (RLOO) on a different and much larger dataset (40K DeepScaleR problems), and evaluated here out-of-distribution on MATH-500.
Under those conditions it reaches $75.6\%$ against the label-free vector's $72.0\%$.
The two findings are consistent once the confound is accounted for: a label-free reward is not what costs accuracy relative to a labeled one at a matched recipe, but a labeled vector trained on $80\times$ more problems with a different optimizer can still be stronger on this evaluation.
We report both rather than only the favourable one, and we do not claim that label-free bias-only TTRL outperforms labeled bias steering in general; only that, holding data and optimizer fixed, the label source is not the binding constraint.

A per-sample confusion matrix shows the two vectors agree on \emph{which} problems they fix far more than chance would predict: 129/500 fixed by both, 26 labeled-only, 17 label-free-only, 4 regress under both, a Jaccard overlap of $129/(129{+}26{+}17){=}0.75$.
Spot-checking the disagreements shows ordinary arithmetic slips on the losing side, not a different problem-solving strategy.

Cosine similarity between the two 28-layer, 3584-dimensional vectors shows a layer-wise trend: $+$0.22 to $+$0.33 in the earliest layers (0--19\% depth, far from the $\approx$0 expected for two random vectors, though the overall cosine of $+$0.172 is itself modest), falling to $+$0.01 and finally $-$0.035 by the final layer, evidence the two vectors converge toward similar computation early and grow checkpoint-specific late; Appendix~\ref{app:rq4-extra} looks at the same divergence through a complementary logit-lens.

\section{Logit-Lens Comparison: Full Detail}
\label{app:rq4-extra}

This appendix reports the full per-depth logit-lens token comparison against labeled bias steering, complementing the vector-comparison detail in Appendix~\ref{app:labelcompare-extra} with additional model/task comparisons.

The full semantic progression on Qwen2.5-7B (label-free, MATH-500) shows layers 0--$\sim$61\% depth are semantic noise; ``solve''/``solving'' vocabulary emerges at $\sim$68--71\% depth; \texttt{\textbackslash boxed\{\}} becomes the top cosine-similarity token at $\sim$86--89\% depth; and causal connectors (``Therefore'', ``Thus'', ``First'') dominate the final layers.
This task- and teacher-specific character is confirmed across three further comparisons, each varying only model or task: our labeled Sinii-reproduction run shows a distinct late-layer vocabulary, detailed below (Table~\ref{tab:labelcomparetokens}); the Qwen2.5-Math-7B variant shows a different step-prefix cluster; and the COT-format MathVista and MMAU vectors (main paper \S\ref{sec:rq3-audit}) show modality-specific late-layer clusters rather than math-specific vocabulary.

\begin{table}[h]
\centering
\scriptsize \setlength{\tabcolsep}{3pt} \renewcommand{\arraystretch}{0.82}
\begin{tabular}{p{0.12\linewidth}p{0.42\linewidth}p{0.38\linewidth}}
\toprule
Depth & Label-free top tokens & Labeled (Sinii-repro) top tokens \\
\midrule
74\% & \texttt{solve}, \texttt{solving} & noisy \\
81\% & \texttt{Solution}, \texttt{solve} & \texttt{simpl}, \texttt{simplify} \\
89\% & \texttt{solution}, \texttt{boxed} & noisy \\
96\% & \texttt{solve}, \texttt{boxed}, \texttt{Therefore/Thus} & noisy (non-English) \\
100\% (final) & noisy/formatting & \texttt{To}, \texttt{\textbackslash}, \texttt{of} \\
\bottomrule
\end{tabular}
\renewcommand{\arraystretch}{1.0} \caption{Representative logit-lens tokens, label-free vs.\ labeled Sinii-reproduced \citep{sinii2025steering} (both Qwen2.5-7B, all-layer \texttt{down\_proj.bias}), at matched depths.}
\label{tab:labelcomparetokens}
\end{table}

The two vectors' tokens are disjoint at every depth in Table~\ref{tab:labelcomparetokens}, consistent with the near-zero/negative cosine similarity at these depths (Appendix~\ref{app:labelcompare-extra}).
Label-free's late-layer tokens cluster around \texttt{Solution}/\texttt{solve}, \texttt{\textbackslash boxed\{\}}, and step-connectives; labeled's cluster around \texttt{simplify}-family tokens at 81\% depth and a step-introduction token plus a raw LaTeX-continuation marker at the final layer, evidence of different lexical biases despite similar reasoning trajectories; we do not take this as proof of what each vector computes.

\paragraph{Format changes what the vector encodes, not just how well it performs.}
The boxed-format MathVista vector (delayed-cosine, the checkpoint behind the main paper's MathVista headline of 65.40\% in Table~\ref{tab:main_results}) gives a final-layer cluster qualitatively different from every other checkpoint: its top token is the digit \texttt{0} (cosine 0.24), alongside other digits, MCQ letters, and a backslash, literally the \texttt{\textbackslash boxed\{\}} answer alphabet, rather than the abstract connectives or negation tokens seen in the text MATH-500 and COT-format MathVista checkpoints.

\section{Additional Theoretical Validation Details}
\label{app:theory-extra}

This appendix expands the main paper's theoretical grounding (Method \S\ref{sec:theory}) with full proof sketches, kept here for transparency and space rather than in the main text.

\subsection{Full Statements and Proof Sketches}
\label{app:theory-proofs}

\paragraph{Proposition 1 setup.}
For an unlabeled input $x$, let the model induce an answer distribution $p(a\mid x)$ over $K$ possible answers, with $a^\star$ the correct answer.
Sample $G$ i.i.d.\ rollouts $A_1,\ldots,A_G\sim p(\cdot\mid x)$ and let $N_a=\sum_{g=1}^G \mathbf{1}[A_g=a]$; the majority-vote pseudo-label is $\hat a_G=\arg\max_a N_a$ (ties count as failures, a conservative choice).
Define the answer margin $\Delta(x)=p(a^\star\mid x)-\max_{a\neq a^\star}p(a\mid x)$ and assume $\Delta(x)>0$: the correct answer is the model's unique mode.

\begin{proposition}[Pseudo-Label Reliability]
If $\Delta(x)>0$,
\[
\Pr[\hat a_G \neq a^\star] \;\le\; (K-1)\exp\!\left(-\frac{G\Delta(x)^2}{2}\right).
\]
Pseudo-label failure under majority voting decays exponentially in $G$ for positive-margin problems.
\end{proposition}

\begin{proof}
\begin{itemize}[leftmargin=1.4em]
\item For any $a\neq a^\star$, let $Z_g=\mathbf{1}[A_g=a^\star]-\mathbf{1}[A_g=a]\in[-1,1]$.
\item Then $\mathbb{E}[Z_g]=p(a^\star\mid x)-p(a\mid x)\ge\Delta(x)>0$ and $\sum_g Z_g = N_{a^\star}-N_a$.
\item Answer $a$ can only tie or beat $a^\star$ if $\frac{1}{G}\sum_g Z_g\le 0$.
\item Hoeffding's inequality bounds this event by $\exp(-G\Delta(x)^2/2)$.
\item A union bound over the $K-1$ incorrect answers gives the proposition.
\end{itemize}
\end{proof}

\paragraph{Theorem 1 setup.}
Let $J(w)$ be the RL objective (the same advantage-weighted log-probability objective formalized in Algorithm~\ref{alg:ttrl}) and $g=\nabla J(w)$ the full gradient.
Restricting adaptation to a bias subspace $\mathcal{S}$ with orthogonal projection $P_\mathcal{S}$, only the component $P_\mathcal{S}g$ is reachable; define the accessible gradient energy $E_\mathcal{S}=\lVert P_\mathcal{S}g\rVert^2$ and the restricted update $w^+=w+\eta P_\mathcal{S}g$.

\begin{theorem}[Restricted-Subspace Trainability]
If $J$ is locally $L$-smooth,
\[
J(w^+)-J(w)\;\ge\; \eta\Big(1-\tfrac{L\eta}{2}\Big)\lVert P_\mathcal{S}g\rVert^2,
\]
and for $0<\eta\le 1/L$, $J(w^+)-J(w)\ge \tfrac{\eta}{2}\lVert P_\mathcal{S}g\rVert^2$ (the bound reported in the main paper).
\end{theorem}

\begin{proof}
\begin{itemize}[leftmargin=1.4em]
\item $L$-smoothness gives $J(w+d)\ge J(w)+\langle \nabla J(w),d\rangle - \tfrac{L}{2}\lVert d\rVert^2$.
\item Set $d=\eta P_\mathcal{S}g$ and use $\langle g,P_\mathcal{S}g\rangle=\lVert P_\mathcal{S}g\rVert^2$ (a projection is idempotent and self-adjoint) to get
\[
J(w^+)-J(w)\ge \eta\lVert P_\mathcal{S}g\rVert^2 - \tfrac{L\eta^2}{2}\lVert P_\mathcal{S}g\rVert^2,
\]
which is the stated bound.
\item The $\eta\le1/L$ case follows by bounding its coefficient below by $1/2$.
\item Writing $\rho_\mathcal{S}=\lVert P_\mathcal{S}g\rVert^2/\lVert g\rVert^2$ for the accessible-gradient fraction, $J(w^+)-J(w)\ge\eta(1-\tfrac{L\eta}{2})\rho_\mathcal{S}\lVert g\rVert^2$: local improvement scales with $E_\mathcal{S}$, not with $\dim(\mathcal{S})$.
\end{itemize}
\end{proof}